\documentclass{article}

\usepackage[preprint]{neurips_2026}

\usepackage[T1]{fontenc}
\usepackage[utf8]{inputenc}
\usepackage{lmodern}
\usepackage{url}
\usepackage{booktabs}
\usepackage{microtype}
\usepackage{xcolor}
\usepackage{amsmath,amssymb,amsthm,mathtools}
\usepackage{bm}
\usepackage{enumitem}
\usepackage{graphicx}
\usepackage{tabularx,array}
\usepackage[hidelinks]{hyperref}
\usepackage[capitalize,nameinlink,noabbrev]{cleveref}
\usepackage{caption}

\newcommand{\E}{\mathbb{E}}
\newcommand{\Hh}{\mathrm{H}}
\newcommand{\I}{\mathrm{I}}
\newcommand{\supp}{\operatorname{supp}}
\newcommand{\Reach}{\operatorname{Reach}}
\newcommand{\Emp}{\operatorname{Emp}}
\newcommand{\Tcal}{\mathcal{T}}
\newcommand{\Zcal}{\mathcal{Z}}
\newcommand{\Xcal}{\mathcal{X}}
\newcommand{\Acal}{\mathcal{A}}
\newcommand{\Ocal}{\mathcal{O}}
\newcommand{\Ccal}{\mathcal{C}}

\newcommand{\Vcal}{\mathcal{V}}
\newcommand{\Ycal}{\mathcal{Y}}

\newtheorem{theorem}{Theorem}
\newtheorem{lemma}{Lemma}
\newtheorem{proposition}{Proposition}
\newtheorem{corollary}{Corollary}
\newtheorem{definition}{Definition}
\newtheorem{assumption}{Assumption}
\newtheorem{algorithm}{Algorithm}
\theoremstyle{remark}
\newtheorem{remark}{Remark}

\title{Prediction and Empowerment: A Theory of Agency through Bridge Interfaces}

\author{Richard Csaky\\\texttt{richard.csaky@gmail.com}}

\begin{document}
\maketitle

\begin{abstract}
We study agency under partial observability in deterministic physical or simulated worlds, where apparent randomness arises from uncertainty over initial conditions, fixed law bits, and unrolled exogenous noise. We model sensing and actuation as bridge interfaces split between agent-controlled parameters and environment-controlled channel state, inducing a deterministic POMDP through a prior over latent microstates and many-to-one observation coarsening. Within this framework, we prove a separation between prediction, compression, and empowerment. Perfect prediction can be achieved either by identifying the hidden quotient relevant to the target family or by overwrite control that makes the future target action-determined; high empowerment alone is insufficient. Under refinable interfaces and sufficient memory, action-conditioned observation-compression progress reduces posterior uncertainty about the latent quotient, and when refinement requires steering world-side channel conditions, this creates target-conditioned interface empowerment. A bit-string specialization with a conserved information budget makes the resulting tradeoff explicit: prediction by identification requires internal capacity at least the relevant latent entropy, whereas overwrite control requires terminal action capacity over the controlled quotient. For modern AI agents, the results suggest a design principle rather than a theorem of inevitability: objectives should distinguish hidden-state identification, interface refinement, task-relevant controllability, and mere overwrite or distractor control. Human--AI alignment is partly an interface-design problem, where the relevant bridge is between human intent, agent internal state, external tools, and world-side channel conditions\footnote{This is a working draft. Feedback and criticism is most welcome.}.
\end{abstract}

\section{Introduction}

A general agent (e.g. human) observes and controls the world through bridges (interfaces). These include sensors, grippers, cameras, retrieval systems, prompts, API permissions, social protocols, and memory. Each bridge has an agent-owned side, such as a query, actuator mode, prompt, or protocol. It also has an environment-owned side, such as occlusion, contact geometry, authorization, rate limits, or human context. Both sides determine whether the agent can learn or control a hidden variable. Standard POMDP notation represents this dependence in a kernel. It leaves the interface variables implicit, which makes failures of prediction, control, and empowerment harder to diagnose.

The motivating hypotheses are H1 and H2. H1 says future-observation compression, hidden-state compression, observation control, state control, and empowerment become aligned. H2 says this alignment drives an agent toward a lossless interface or absorption of the relevant environment. These hypotheses require explicit qualifications. Empowerment can be spent on a settable distractor. Prediction can succeed by erasing latent distinctions. Observation control can diverge from state control through a lossy display. A strong version must specify the conditions under which one objective can substitute for another in decision making. Pointwise closeness of two information quantities supplies too little information.

This paper makes H1 a decision-theoretic statement. For any option class, the obstruction to using one H1 objective as a surrogate for another is the \emph{oscillation} of their difference over that class. Bridge-gap components bound this oscillation in bits. A small pointwise gap at a single policy gives too little control. The gap must be uniformly small over the policies being optimized. This correction turns entropy bookkeeping into a guarantee about optimizer transfer.

\paragraph{Contributions.}
The paper represents any finite-horizon stochastic POMDP as a deterministic latent bridge-POMDP by including exogenous noise in the latent variable (\cref{sec:model}). It proves universal bridge-gap inequalities and the tight uniform regret-transfer theorem (\cref{sec:bridge_gap}). It proposes \emph{Bridge-Gap Pursuit} (BGP), a model-based intrinsic-control algorithm whose reward is the decrease of a bridge potential. Exact finite-horizon BGP minimizes terminal bridge gap by dynamic programming (\cref{sec:refinement_algorithm}). The paper also implements exact finite benchmarks showing arbitrary separations from ungated empowerment, one-step information gain/EFE, prediction-loss optimization, and coarse training return (\cref{sec:benchmarks}). It formalizes H2 as a prediction-absorption bound and derives implications for human intelligence and AGI interfaces (\cref{sec:h2,sec:design}).

\section{Bridge-interface POMDPs}
\label{sec:model}

Fix a horizon $T$. Let
\begin{equation}
  Z=(X_0,\Lambda,\Xi_{0:T})\in\Zcal
\end{equation}
be the latent environment variable, including the initial environment state, optional law bits, and any exogenous transition/observation noise needed over the horizon. A prior $p_0$ over $Z$ is the source of uncertainty. A randomized policy may also use a private seed $\Sigma$, sampled independently of $Z$. All deterministic-map statements below are interpreted either for deterministic policies or conditional on a realized seed $\Sigma=s$. Equivalently, the closed-loop history is a deterministic function of $(Z,\Sigma)$, but losslessness for environment quotients is always about $Z$ and never requires recovering the agent's private randomization. The finite deterministic model below covers finite stochastic POMDPs at a fixed horizon by unrolling their random seeds. Targets may quotient out future exogenous noise if only predictable structure is relevant.

The timing convention is observe--act--observe. The decision-time history is
\begin{equation}
  H_t=(O_0,A_0,O_1,A_1,\ldots,A_{t-1},O_t),\qquad t=0,\ldots,T,
\end{equation}
and the agent state $M_t$ is an $H_t$-measurable retained state. At decision time $t<T$, the policy maps $M_t$ to the intended action $A_t$ and the agent-owned bridge settings used for the ensuing transition and observation. Sensing and actuation pass through bridge variables
\begin{equation}
  \phi_{t+1}=(\phi_{t+1}^S,\phi_{t+1}^X),\qquad
  \kappa_t=(\kappa_t^S,\kappa_t^X).
\end{equation}
The superscript $S$ denotes agent-owned settings. The superscript $X$ denotes environment-owned channel conditions determined by the current or next world state and law/noise bits. With an initial observation $O_0$ generated by a fixed initial sensor setting, the closed-loop step is
\begin{align}
  \kappa_t^X &= k_t^X(X_t,\Lambda,\Xi_t), &
  U_t &= \alpha_t(A_t,\kappa_t^S,\kappa_t^X), \notag\\
  X_{t+1} &= f_t(X_t,U_t,\Lambda,\Xi_t), &
  \phi_{t+1}^X &= \ell_{t+1}^X(X_{t+1},\Lambda,\Xi_{t+1}), \label{eq:model_dyn}\\
  O_{t+1} &= h_{t+1}(X_{t+1},\phi_{t+1}^S,\phi_{t+1}^X,\Xi_{t+1}), &
  M_{t+1} &= g_t(M_t,A_t,O_{t+1}). \label{eq:model_obs}
\end{align}
The transcript is
\begin{equation}
  \Tcal_T=(O_0,A_0,O_1,A_1,\ldots,A_{T-1},O_T).
\end{equation}
For every deterministic policy, and for every randomized policy conditional on $\Sigma=s$, there are deterministic maps
\begin{equation}
  \Tcal_T=G_{T,s}^\pi(Z),\qquad X_T=F_{T,s}^\pi(Z),\qquad M_T=S_{T,s}^\pi(Z). \label{eq:maps_v4}
\end{equation}
When the seed is irrelevant or the policy is deterministic, the subscript $s$ is omitted.

\begin{definition}[Target quotient and losslessness]
For a family $\mathcal R$ of deterministic targets $R:\Zcal\to\mathcal Y_R$, define $z\sim_{\mathcal R}z'$ iff $R(z)=R(z')$ for all $R\in\mathcal R$. Let $Q_{\mathcal R}=q_{\mathcal R}(Z)$ be the quotient. A transcript is lossless for $Q$ if $\Hh(Q\mid\Tcal_T)=0$. It is microstate-lossless if $Q=Z$ on $\supp(p_0)$.
\label{def:target_quotient}\label{def:lossless}
\end{definition}

This quotient notation prevents overclaiming. An agent should identify the authorized distinctions needed for a target family. A default demand for every microscopic or private fact would overreach. The limiting case, used for absorption statements, is a task-separating family whose quotient equals the latent environment bits under consideration.

\section{The bridge-gap theorem}
\label{sec:bridge_gap}

Let $Q$ be a hidden-state quotient and $W$ a future-observation quotient. Let $V$ be a terminal environment quotient and $\widetilde V$ the terminal observation quotient through which control is evaluated. For control-capacity statements, condition on a fully specified deterministic rollout context $\chi$ containing the initial physical state and all law/noise bits that must be fixed for deterministic reachability. If one conditions only on the physical state $X_0$, the same mutual-information inequalities hold for the induced stochastic channel, but the log-reachability identity in \cref{lem:emp_reach} should be read with $\chi$ in place of $X_0$. Write
\begin{equation}
 C_V(\chi)=\max_{p(a_{0:T-1})}\I(A_{0:T-1};V\mid \chi),\quad
 C_{\widetilde V}(\chi)=\max_{p(a_{0:T-1})}\I(A_{0:T-1};\widetilde V\mid \chi).
\end{equation}
These are open-loop action-sequence capacities. Closed-loop variants can be obtained by replacing action sequences with admissible policies or policy seeds; the bridge-gap inequalities below are unchanged.
Define the bridge-gap components
\begin{align}
 \Delta_{QW}&=\max\{\Hh(Q\mid W),\Hh(W\mid Q)\}, &
 \Delta_{\rm sense}&=\Hh(Q\mid\Tcal_T),\label{eq:gaps1}\\
 \Delta_{V\widetilde V}(\chi)&=\sup_p\max\{\Hh_p(V\mid\widetilde V,\chi),\Hh_p(\widetilde V\mid V,\chi)\}, &
 \Delta_{\rm act}(\chi)&=\log |\supp(V\mid\chi)|-C_V(\chi).\label{eq:gaps2}
\end{align}
The supremum ranges over action-sequence distributions under context $\chi$. The total bridge gap is any nonnegative weighted sum of these components. Positive weights leave the zero set invariant.

\begin{theorem}[Universal bridge-gap bounds]
\label{thm:bridge_gap}
For any finite bridge-POMDP, deterministic policy or policy seed, horizon, transcript, and quotients as above,
\begin{align}
 -\Hh(W\mid Q)&\le \I(Q;\Tcal_T)-\I(W;\Tcal_T) \le \Hh(Q\mid W), \label{eq:compression_gap}\\
 -\sup_p\Hh_p(\widetilde V\mid V,\chi)&\le C_V(\chi)-C_{\widetilde V}(\chi) \le \sup_p\Hh_p(V\mid\widetilde V,\chi). \label{eq:control_gap}
\end{align}
Moreover,
\begin{equation}
 1-\frac{\I(Q;\Tcal_T)}{\Hh(Q)}=\frac{\Delta_{\rm sense}}{\Hh(Q)},\qquad
 1-\frac{C_V(\chi)}{\log |\supp(V\mid\chi)|}=\frac{\Delta_{\rm act}(\chi)}{\log |\supp(V\mid\chi)|}
\end{equation}
whenever the denominators are nonzero. If the relevant bridge-gap components vanish, the corresponding normalized quantities--future-observation compression, hidden-state compression, observation control, state control, and deterministic empowerment on $V$--attain their lossless or full-capacity value. Conversely, joint maximality on the evaluated quotients forces the associated nonnegative deficits in \cref{eq:gaps1,eq:gaps2} to vanish.
\end{theorem}

The theorem is pointwise. It describes how two quantities differ at a fixed policy or action distribution. Optimization requires a uniform statement. Let $\mathcal U$ be any finite policy or option class and let $J_i,J_j:\mathcal U\to[0,1]$ be any two normalized H1 objectives, for example normalized hidden-state compression, future-observation compression, state-control capacity, observation-control capacity, or quotient empowerment. Define the objective-difference oscillation
\begin{equation}
  \Omega_{ij}(\mathcal U)=\sup_{u,v\in\mathcal U}\bigl[(J_i(u)-J_j(u))-(J_i(v)-J_j(v))\bigr]. \label{eq:oscillation}
\end{equation}

\begin{theorem}[Tight uniform regret transfer]
\label{thm:regret_transfer}
If $\hat u$ is $\eta$-optimal for $J_i$ over $\mathcal U$, then its regret for $J_j$ is at most
\begin{equation}
  \max_{u\in\mathcal U}J_j(u)-J_j(\hat u)\le \eta+\Omega_{ij}(\mathcal U). \label{eq:regret_transfer}
\end{equation}
The bound is tight for arbitrary finite objective pairs. Thus the exact obstruction to surrogate optimization is $\Omega_{ij}(\mathcal U)$. Bridge-gap bounds are sufficient diagnostics: if they imply a uniform absolute bound $|J_i(u)-J_j(u)|\le\varepsilon$ over $\mathcal U$, then $\Omega_{ij}(\mathcal U)\le 2\varepsilon$; if they directly bound the range width of $J_i-J_j$ by $\varepsilon$, then $\Omega_{ij}(\mathcal U)\le\varepsilon$.
\end{theorem}

This theorem addresses a common overclaim. A small gap at the selected policy gives too little support for surrogate optimization. The whole option class must have small objective-difference oscillation. Zero bridge gap for every option in $\mathcal U$ implies $\Omega_{ij}=0$ and exact optimizer transfer. If the uniform normalized absolute gap is at most $\varepsilon$, an $\eta$-optimizer of one H1 objective is an $(\eta+2\varepsilon)$-optimizer of the other; if the bridge gap bounds the range width by $\varepsilon$, the regret transfer is $(\eta+\varepsilon)$. A large component identifies a specific failure mode. The possibilities are wrong quotient, insufficient evidence, lossy observation of controlled state, or insufficient authority.

\section{Prediction, identification, and empowerment separations}
\label{sec:prediction_empowerment}

\begin{theorem}[Prediction is fiber collapse]
\label{thm:fiber}
Fix $\pi,T$ and either assume $\pi$ is deterministic or condition on a realized private seed $\Sigma=s$. Let $Y=R_{T,s}^\pi(Z)$ be any finite deterministic target. The Bayes-optimal log-loss for predicting $Y$ from $\Tcal_T$ is $\Hh(Y\mid\Tcal_T,\Sigma=s)$. Perfect prediction holds exactly when $R_{T,s}^\pi$ is constant on every positive-prior fiber of $G_{T,s}^\pi$. Microstate-losslessness holds exactly when $G_{T,s}^\pi$ is injective on $\supp(p_0)$.
\end{theorem}

\begin{corollary}[General prediction pressure]
\label{cor:general_prediction}
Perfect prediction for all targets in a family $\mathcal R$ is equivalent to $\Hh(Q_{\mathcal R}\mid\Tcal_T)=0$. If $\mathcal R$ is task-separating, perfect prediction for all targets implies microstate-losslessness.
\end{corollary}

For a fully specified deterministic rollout context $\chi$, deterministic horizon-$T$ open-loop empowerment over quotient $V$ is
\begin{equation}
  \Emp_T^V(\chi)=\max_{p(a_{0:T-1})}\I(A_{0:T-1};V\mid \chi).
\end{equation}
Let $\Reach_T^V(\chi)$ be the set of reachable values of $V$ under action sequences from that context.
\begin{lemma}[Empowerment equals log reachability]
\label{lem:emp_reach}
If $V$ is a deterministic function of the action sequence conditional on $\chi$, then $\Emp_T^V(\chi)=\log_2|\Reach_T^V(\chi)|$.
\end{lemma}

\begin{proposition}[Arbitrary separations]
\label{prop:separation}
For any integers $m,n\ge1$, finite bridge-POMDPs can realize each of the following separations.
\begin{enumerate}[label=(\roman*)]
  \item Zero empowerment with perfect identification of a uniform $n$-bit latent.
  \item $m$ bits of empowerment with $\I(Z;\Tcal_1)=0$.
  \item Perfect prediction of a terminal target with $\I(Z;M_T)=0$ because the policy overwrites the target before predicting it.
\end{enumerate}
\end{proposition}

These examples are common reasons intrinsic rewards fail. A camera can classify its own display and leave the task object unseen. A robot can control a light switch and leave the hidden latch state unresolved. A language agent can make a later answer predictable by forcing the user into a narrow protocol that bypasses intent inference. The bridge gap says which of these happened in bits.

\section{Refinement and Bridge-Gap Pursuit}
\label{sec:refinement_algorithm}

The gap also gives minimal resource requirements. Let $B$ be any additional transcript generated by a refined sensor, query, tool call, or experiment after the current transcript $\Tcal$.
\begin{theorem}[Exact missing sensing bits]
\label{thm:missing_bits}
If $\Hh(Q\mid\Tcal,B)=0$, then $\Hh(B\mid\Tcal)\ge\Hh(Q\mid\Tcal)$. This lower bound is tight. The sufficient statistic $B=Q$ has conditional entropy $\Hh(Q\mid\Tcal)$ and closes the sensing gap.
\end{theorem}
Thus $\Hh(Q\mid\Tcal)$ is both a diagnostic and the exact number of additional transcript bits any lossless refinement must supply. The analogous outward deficit in a deterministic rollout context $\chi$ is $\Delta_{\rm act}(\chi)=\log|\supp(V\mid\chi)|-C_V(\chi)$. A deterministic controller with capacity $C_V(\chi)$ can realize only $2^{C_V(\chi)}$ terminal values of $V$, and full authority requires the missing capacity.

A deterministic sensing design $\eta$ induces an experiment $G_\eta:Z\to\Ocal_\eta$. With actions and private seed conditioned on, $O_\eta=G_\eta(Z)$.
\begin{theorem}[Refinement converts observation entropy into latent information]
\label{thm:refinement}
For any deterministic experiment,
\begin{equation}
  \I(Z;O_\eta)=\Hh(O_\eta)=\Hh(Z)-\Hh(Z\mid O_\eta).
\end{equation}
If $\eta'$ refines $\eta$ in the Blackwell sense, $G_\eta=r\circ G_{\eta'}$ on $\supp(p_0)$, then $\I(Z;O_{\eta'})\ge\I(Z;O_\eta)$, with strict inequality under positive-mass strict refinement.
\end{theorem}
The conditioning clause prevents private randomization from masquerading as discovery. Refinement may itself require control. The agent must reach channel states such as line of sight, contact, authorization, or a useful tool state before the informative experiment exists.

\begin{algorithm}[Bridge-Gap Pursuit (BGP)]
\label{alg:bgp}
Maintain a posterior model over the authorized quotient $Q$, world-side bridge states $C=(\phi^X,\kappa^X)$, controllable factors $Y_1,\ldots,Y_d$, observed terminal quotient $\widetilde V$, and task value. Define the bridge potential
\begin{align}
 \Phi(h)=&\ \widehat\Hh(Q\mid h)
 +\lambda_C[\log|\Ccal_Q^\star|-\widehat C_C(h)]_+
 +\lambda_V[\log|\Vcal_Q|-\widehat C_{V_Q}(h)]_+  \notag\\
 &+\lambda_O\widehat\Hh(V_Q\mid\widetilde V,h)
 +\lambda_D\sum_i(1-\widehat g_i(h))\widehat C_{Y_i}(h), \label{eq:bgp_potential}
\end{align}
where $\Ccal_Q^\star$ are channel states enabling informative experiments about $Q$, and $V_Q$ is the task-relevant controlled quotient. The counterfactual relevance gate is defined from intervention value rather than from the assigned variable as evidence:
\begin{align}
  \widehat\Delta_i^Q(h)&=\sup_{y\in\Ycal_i}\Bigl[\widehat\Hh(Q\mid h)-\widehat{\E}\bigl[\widehat\Hh(Q\mid H_T)\mid h,do(Y_i=y)\bigr]\Bigr]_+, \notag\\
  \widehat\Delta_i^R(h)&=\sup_{y,y'\in\Ycal_i}\left|\widehat{\E}[R_{t:T}\mid h,do(Y_i=y)]-\widehat{\E}[R_{t:T}\mid h,do(Y_i=y')]\right|, \notag\\
  \widehat g_i(h)&=\mathbf 1\{\widehat\Delta_i^Q(h)+\widehat\Delta_i^R(h)>\tau\}. \label{eq:relevance_gate}
\end{align}
The intervention $do(Y_i=y)$ does not make the assigned value $Y_i$ evidence about $Q$; it is relevant only if changing $Y_i$ causally improves future evidence about $Q$ or changes future task value.
Plan over options with reward
\begin{equation}
 r_t^{\rm BGP}=r_t^{\rm task}+\beta\bigl(\Phi(H_t)-\Phi(H_{t+1})\bigr). \label{eq:bgp_reward}
\end{equation}
\end{algorithm}

The last term assigns opportunity cost to controllable factors that are independent of the target quotient and downstream value, especially under scarce control budget. The channel term gives delayed interface actions value before they reveal information. The observation-loss term prevents hidden state control from being mistaken for verifiable control.

\begin{theorem}[BGP is exact bridge-potential shaping]
\label{thm:bgp_shaping}
For any finite-horizon bridge-POMDP with exact $\Phi$, every policy $\pi$ satisfies
\begin{equation}
\E_\pi\sum_{t=0}^{T-1}\bigl(\Phi(H_t)-\Phi(H_{t+1})\bigr)=\Phi(H_0)-\E_\pi\Phi(H_T). \label{eq:telescoping}
\end{equation}
Thus, when $r^{\rm task}=0$, an optimal finite-horizon BGP policy is exactly a policy minimizing expected terminal bridge gap. With task reward, BGP solves the Lagrangian objective $\E[\sum_t r_t^{\rm task}-\beta\Phi(H_T)]$ up to the constant $\beta\Phi(H_0)$. If a zero-gap terminal history is reachable within horizon $T$, exact dynamic-programming BGP reaches one whenever doing so is optimal under this Lagrangian.
\end{theorem}

BGP's substantive algorithmic claim is the choice of potential and its decomposition into bridge deficits. It can be implemented by dynamic programming in exact finite domains, by model-predictive control or tree search in learned world models, or by training a policy on the shaped reward. Approximate implementations require separate assumptions about the learned estimators. The learnable losses are
\begin{equation}
\mathcal L_{\rm BGP}=\mathcal L_{\rm world}+\alpha_Q\mathcal L_Q+
\alpha_C\mathcal L_C+\alpha_V\mathcal L_V+
\alpha_R\mathcal L_{\rm rel},
\end{equation}
where the heads estimate posterior ambiguity, channel reachability, controllability, observation loss, and counterfactual relevance. The algorithm changes what the agent models and rewards. It uses the existing outer RL optimizer.

\begin{proposition}[Algorithmic separations with arbitrary gaps]
\label{prop:algorithmic_separations}
For every $n,m\ge1$, finite deterministic bridge-POMDPs can realize each of the following failures.
\begin{enumerate}[label=(\roman*)]
  \item Ungated empowerment chooses an $m$-bit settable distractor and leaves $\Hh(Q\mid H_T)=n$.
  \item One-step information gain or one-step expected-free-energy assigns equal value to the unique preliminary channel action needed for later sensing and can leave $\Hh(Q\mid H_T)=n$.
  \item Terminal prediction loss chooses overwrite, predicts its terminal target perfectly, and leaves $\Hh(Q\mid M_T)=n$.
  \item Coarse training return leaves a refined quotient with $n$ unresolved bits.
\end{enumerate}
Exact BGP selects the zero-gap option in all four families when it is reachable.
\end{proposition}

\section{Bridge budget for evidence, influence, and plasticity}
\label{sec:budget}

The bridge gap measures whether the same quotient is available to prediction and control. A separate issue is whether a finite interface can simultaneously carry evidence inward and influence outward. Empowerment measures how much future state can be influenced by action \citep{Klyubin2005Empowerment}. Plasticity is the inward mirror and measures how much future action can be influenced by observation \citep{Abel2025Plasticity}. When both flows use a shared finite bridge transcript, they obey a capacity budget.

For this section only, use the standard feedback-channel indexing $A^T=(A_1,\ldots,A_T)$ and $O^T=(O_1,\ldots,O_T)$, where $A_t$ is chosen after $O^{t-1}$ and before $O_t$; the initial observation in the POMDP convention can be absorbed into the initial condition. Define directed information
\begin{align}
  \I(A^T\to O^T)&=\sum_{t=1}^T \I(A^t;O_t\mid O^{t-1}),\\
  \I(O^{T-1}\to A^T)&=\sum_{t=1}^T \I(O^{t-1};A_t\mid A^{t-1}).
\end{align}
The first term measures realized outward influence through the interface. The second measures realized inward evidence that changes action.

\begin{theorem}[Finite shared-bridge budget]
\label{thm:di_budget_main}
Assume the joint law is causal and all statistical dependence between $A^T$ and $O^T$ is mediated by a specified finite bridge transcript $B$ satisfying $\Hh(B)\le b$, so $A^T\perp O^T\mid B$. The variable $B$ is the physical or informational interface transcript being bounded; it should not include arbitrary agent memory or private randomness unless those degrees of freedom are actually transmitted through the interface. Then
\begin{equation}
  \I(A^T\to O^T)+\I(O^{T-1}\to A^T)=\I(A^T;O^T)\le b .
\end{equation}
\end{theorem}

This theorem is a bottleneck theorem for a specified bridge. It is relevant for H1 and BGP because a narrow interface can spend capacity on easy outward control, leaving little capacity for evidence. It can also spend capacity on passive evidence, leaving little realized influence. A unified agent must widen the bridge, split inward and outward channels, or internalize the relevant state. This is why BGP rewards channel reachability over raw empowerment. A controllable bit is useful when it reduces the bridge gap for the target quotient or opens a better experiment.

\section{H2: unification and absorption from prediction}
\label{sec:h2}

\begin{definition}[Unification]
For latent quotient $Q$, terminal state quotient $V$, and observed terminal quotient $\widetilde V$, a policy is unified at horizon $T$ if the bridge gap for $(Q,V,\widetilde V)$ is zero. Equivalently, $Q$ is transcript-identifiable, $V$ is fully controllable, and $\widetilde V$ is lossless for $V$ on the reachable set.
\label{def:unification}
\end{definition}

The earlier absorption statement can be strengthened and made less assumption-heavy by stating what prediction forces into agent-owned state. Let $M_T$ be all retained agent-owned state, including memory, parameters updated during the episode, scratchpad, owned interface settings, and any maintained external memory counted as part of the agent.

\begin{theorem}[Prediction-absorption bound]
\label{thm:absorption_trichotomy}
Fix a deterministic policy or condition on a realized private seed $\Sigma=s$. Let $Q=q(Z)$ be an initial environment quotient and $V=F_{T,s}^\pi(Z)$ a terminal quotient. If $V$ is perfectly predictable from agent state, $\Hh(V\mid M_T)=0$, then
\begin{equation}
  \I(Q;M_T)\ge \I(Q;V),\qquad
  \Hh(Q\mid M_T)\le \Hh(Q\mid V). \label{eq:absorb_bound}
\end{equation}
The bound leaves three canonical, possibly overlapping mechanisms. \emph{Identification.} If $V$ preserves $Q$ ($\Hh(Q\mid V)=0$), then $\Hh(Q\mid M_T)=0$ and the agent has absorbed $Q$. \emph{Overwrite/collapse.} If the policy maps many $Q$ values to the same $V$, the unabsorbed part is at most the collapsed entropy $\Hh(Q\mid V)$. \emph{Omission.} If the target quotient omits some distinctions in $Q$, prediction can ignore those distinctions.
\end{theorem}

\begin{corollary}[Memory and empowerment consequences]
\label{cor:bit_costs_main}\label{thm:bit_budget}
If $\Hh(Q\mid M_T)=0$, then $\log|\mathcal M_T|\ge\Hh(Q)$. If $V$ has $m$ reachable terminal values under deterministic authority from context $\chi$, then deterministic empowerment over $V$ is $\log m$. Full authority over $V$ is equivalent to $\Emp_T^V(\chi)=\log|\supp(V\mid\chi)|$.
\end{corollary}

This is the operational H2 used here. Joint maximization of the H1 objectives can drive the bridge gap to zero when zero gap is feasible and the optimized objectives have small difference oscillation over the relevant policy class. Zero gap is unification. Absorption follows when the unified target family is task-separating or when the terminal quotient preserves the relevant environment distinctions. Overwrite can make prediction perfect with unresolved latent distinctions. Task omission makes absorption unnecessary and potentially undesirable. When $M_T$ contains the target quotient and the law bits needed for future rollout, the future coupled trajectory under a deterministic policy is a function of agent-owned state. Observations then add zero information about that quotient beyond the self-model.

\section{Implications for human intelligence and AGI design}
\label{sec:design}

The theory leaves human objective functions open. It predicts a developmental invariant. Intelligence expands the class of task quotients for which bridge gap can be reduced. Gaze, posture, reaching, locomotion, haptic exploration, tool use, imitation, language, and social protocol learning are ways to make better experiments and control channels reachable. In the notation above, development should increase $\I(Q_{\mathcal R};M_T)/\Hh(Q_{\mathcal R})$, increase reachability of informative channel states $C\in\Ccal_Q^\star$, and close relevance gates on controllable variables that are causally unrelated to the task. This aligns with active perception, sensorimotor contingencies, active inference, and developmental robotics \citep{Bajcsy1988ActivePerception,OReganNoe2001,Friston2010FEP,CangelosiSchlesinger2015}. It also gives a sharper empirical question. Which latent distinctions still lie in the same transcript fiber, and which bridge state would separate them?

The resulting account of human intelligence is boundary management. Organisms learn outward regularities of the actuation bridge, including the motor commands, gestures, and utterances that reliably change the world. They use that authority to reach channel states with higher sensing value, such as line of sight, contact, shared attention, or clarification. They compress task quotients into memory, skills, and concepts. They externalize stable bridge components through tools, writing, institutions, and social conventions. This ordering is testable. Each stage predicts changes in transcript ambiguity, channel reachability, and controllability under ablations.

For AGI design, BGP suggests an explicit interface layer between a world model and a policy. The layer contains a quotient posterior for authorized task distinctions $Q$, a bridge-state model for world-side channel conditions $C$, a controllability estimator for terminal quotients $V_Q$, and a counterfactual relevance gate that suppresses empowerment over factors independent of $Q$ or future value. A bridge-aware world model represents the source of uncertainty. It distinguishes hidden state, unavailable experiments, unwritable actuator conditions, and distractor variables. This differs from training a monolithic predictor and hoping that useful exploration emerges from prediction error.

For tool-using and language agents, prompts, retrieval schemas, clarifying questions, API permissions, memory writes, tool states, and user feedback are bridge variables. Questions, tool calls, authorization requests, and memory writes have value when they reduce $\Hh(Q\mid H)$, change channel reachability, remove an authorization bottleneck, or prevent recurring quotient ambiguity. Approval and engagement signals should be ignored or penalized when they are controllable and conditionally independent of user intent. A practical evaluation follows directly. Specify the target quotient, inventory bridge ownership, measure transcript ambiguity and terminal control capacity, vary hidden initial state with the transcript fixed, and ablate action bandwidth to distinguish identification from overwrite.

\begin{table}[t]
\centering
\small
\setlength{\tabcolsep}{3.0pt}
\begin{tabularx}{\linewidth}{@{}l>{\raggedright\arraybackslash}X>{\raggedright\arraybackslash}X@{}}
\toprule
Bridge variable & Human examples & AI design examples \\
\midrule
$\phi^S$ & gaze, questioning, exploratory touch & prompt/query format, camera pose, retrieval schema \\
$\phi^X$ & occlusion, lighting, social attention & permission state, tool visibility, database access \\
$\kappa^S$ & grip, posture, practiced motor mode & controller mode, API choice, tool adapter \\
$\kappa^X$ & contact, friction, norm/affordance state & authorization, rate limits, simulator preconditions \\
$Q$ & distinctions needed for perception, skill, social action & authorized latent factors needed for robust downstream tasks \\
\bottomrule\\
\end{tabularx}
\caption{Human development and AI design use the same bridge decomposition.}
\label{tab:human_agi}
\end{table}

\section{Related work}
Bridge interfaces sit between POMDPs and controlled sensing \citep{Kaelbling1998POMDP,Krishnamurthy2016POMDP,HeroCochran2011SensorManagement}, active perception \citep{Bajcsy1988ActivePerception,Aloimonos1988ActiveVision}, Blackwell experiment refinement \citep{Blackwell1953}, information bottlenecks \citep{Tishby2000IB}, empowerment \citep{Klyubin2005Empowerment}, plasticity \citep{Abel2025Plasticity}, active inference \citep{Friston2015EpistemicValue,Millidge2021WhenceEFE}, and world models \citep{HaSchmidhuber2018WorldModels}.

The refinement order is the deterministic special case of comparing statistical experiments by informativeness \citep{Blackwell1953}. The information-gain and compression statements are related to classical information theory and information bottleneck ideas \citep{Shannon1948,CoverThomas2006,Tishby2000IB}. Active inference and expected free energy combine extrinsic and epistemic terms under stochastic generative models \citep{Friston2010FEP,Friston2015EpistemicValue,Friston2016AILearning,Friston2017ProcessTheory,ParrFriston2019GeneralisedFE,DaCosta2020DiscreteAIF,Millidge2021WhenceEFE}. The present paper isolates deterministic identifiability and resource lower bounds.

Empowerment was introduced as an agent-centric control capacity \citep{Klyubin2005Empowerment}. Directed information and transfer entropy provide feedback-aware measures of information flow \citep{Marko1973,Massey1990DI,Schreiber2000,Jiao2013DI}. Recent work on plasticity studies the observation-to-action mirror of empowerment \citep{Abel2025Plasticity}, and the Artificial Agency Program connects curiosity, compression, empowerment, and interface quality as a research agenda \citep{Csaky2026AAP}. The finite bridge budget here is narrower. It is an exact capacity statement when action--observation dependence is mediated by a finite transcript.

\section{Discussion}
\label{sec:benchmarks}\label{sec:related_limitations}

The four benchmark families from \cref{prop:algorithmic_separations} are exact finite enumerations. The default instances use a uniform $n=4$ bit task quotient and an $m=8$ bit distractor. Metrics are terminal transfer success and remaining ambiguity $\Hh(Q\mid H_T)$, both computed exactly from the finite state spaces.

\begin{table}[t]
\centering
\small
\setlength{\tabcolsep}{2.8pt}
\begin{tabularx}{\linewidth}{@{}l>{\raggedright\arraybackslash}Xcc@{}}
\toprule
Benchmark & Failing objective & Baseline success / $\Hh(Q\mid H_T)$ & BGP success / $\Hh(Q\mid H_T)$ \\
\midrule
Settable distractor & ungated empowerment controls irrelevant $D$ & 6.2\% / 4 & 100\% / 0 \\
Delayed sensor & one-step IG/EFE sees zero immediate information & 6.2\% / 4 & 100\% / 0 \\
Inspect-overwrite & terminal prediction loss chooses overwrite & 6.2\% / 4 & 100\% / 0 \\
Quotient transfer & coarse return ignores refined bits & 25.0\% / 2 & 100\% / 0 \\
\bottomrule\\
\end{tabularx}
\caption{Exact finite benchmark results for $n=4,m=8$. The comparison uses exact optimizers for the stated baseline objectives. The failures reflect objective misspecification. Search error is ruled out.}
\label{tab:exact_results}
\end{table}

The gaps can be made arbitrarily large by increasing $m$ to make irrelevant empowerment dominate or by increasing $n$ to make the transfer success of a non-identifying policy $2^{-n}$. A full long-horizon active-inference planner with the correct hidden quotient and interface states can match BGP on some instances. The claim is sharper. Standard epistemic or control objectives need an explicit channel-reachability term, relevance gate, and absorption diagnostic to distinguish policies that differ in bridge gap. BGP adds those missing terms.

For scalable experiments, BGP should be evaluated as an intrinsic reward in a learned latent model. A recurrent world model or transformer summarizes $h_t$. Auxiliary heads estimate $q_\theta(Q\mid h_t)$, bridge reachability $q_\theta(C\mid h_t,u)$, controllability of $V_Q$, observation loss, and counterfactual relevance of factors $Y_i$. The theorem dictates the primary ablations. Removing the relevance gate should reproduce empowerment traps. Removing the channel term should reproduce delayed-sensor failures. Removing the observation-loss term should confuse hidden state control with observable control. Removing the absorption diagnostic should make overwrite and understanding indistinguishable.

\paragraph{Limitations.} The exact theorems are finite-horizon and finite-state. Continuous systems require rate-distortion quotients or discretization. The exact benchmark uses known models. Scalable BGP needs learned estimators of entropy, capacity, relevance, and observation loss. Choosing the authorized quotient $Q$ is a normative task-design problem. These restrictions keep the claim narrow. Agents become more general when they reduce the missing sensing bits, missing authority bits, and quotient mismatch that prevent prediction and control from referring to the same part of the world. A decisive negative result would be an interface-bottleneck benchmark in which the bridge-gap heads are accurate, the target quotient is correctly specified, and BGP systematically underperforms a simpler curiosity, empowerment, or active-inference objective for reasons outside the gap terms. A positive scalable result should report more than return. It should show that the learned policy reduces the predicted and measured bridge gap, that each ablation reintroduces the corresponding failure mode, and that transfer to refined quotients improves because unresolved transcript fibers shrink. A brittle shortcut would fail this check.

This also clarifies the intended empirical impact. BGP is most useful in domains where valuable actions first improve access to experiments, such as moving a camera, obtaining contact, asking a clarifying question, retrieving a document, requesting permission, or configuring a tool. These actions often have low immediate novelty and low immediate reward. They have high value as bridge refinements. The algorithmic claim is that modeling this value explicitly should improve exploration and transfer in such domains and reduce reward hacking through irrelevant controllable variables. For a top-level learning system, this suggests a concrete benchmark criterion. The agent should solve the current task and expose which bridge variable made the task solvable. That requirement is operational. The experimenter can randomize bridge conditions, hide or reveal permissions, alter action bandwidth, and refine the target quotient after training.

\paragraph{LLM disclosure.} Large language model assistance was used during the development and writing of this paper. The authors provided the motivating hypotheses, modeling choices, target claims, examples, and technical direction, and used an LLM interactively to propose formal theorem statements, organize definitions, and draft candidate proofs. The final mathematical statements, assumptions, proofs, experiments, references, and interpretation were checked, revised, and accepted by the authors, who take full responsibility for their correctness and presentation.
\clearpage
\bibliographystyle{plainnat}
\bibliography{neurips_2026_revised_v5}

\clearpage

\appendix

\section{Complete proofs for the main text}
\label{app:proofs}

\subsection{Finite stochastic unrolling}
Any finite stochastic transition or observation kernel over a finite horizon can be realized by sampling independent exogenous variables and applying deterministic inverse-CDF maps. Collecting all environment variables with the initial state and law bits gives $Z=(X_0,\Lambda,\Xi_{0:T})$. Conditional on $Z$ and on a realized private policy seed $\Sigma=s$, the trajectory is deterministic and has the same marginal law over histories as the original stochastic process after marginalizing over $(Z,\Sigma)$. The seed is conditioned on in deterministic-map claims. If the target quotient excludes future exogenous noise, that noise is quotient-collapsed. If it includes future noise, residual uncertainty is correctly counted as an irreducible prediction gap.

\subsection{Proof of \cref{thm:bridge_gap}}
The compression bound follows from the chain rule.
\begin{equation}
\I(Q,W;\Tcal)=\I(W;\Tcal)+\I(Q;\Tcal\mid W)=\I(Q;\Tcal)+\I(W;\Tcal\mid Q).
\end{equation}
Therefore
\begin{equation}
\I(Q;\Tcal)-\I(W;\Tcal)=\I(Q;\Tcal\mid W)-\I(W;\Tcal\mid Q),
\end{equation}
which is at most $\Hh(Q\mid W)$ and at least $-\Hh(W\mid Q)$. For control, fix an action-sequence distribution $p$ and condition throughout on the deterministic rollout context $\chi$. The same identity gives
\begin{equation}
\I(A;V\mid\chi)-\I(A;\widetilde V\mid\chi)=\I(A;V\mid\widetilde V,\chi)-\I(A;\widetilde V\mid V,\chi),
\end{equation}
which lies between $-\Hh(\widetilde V\mid V,\chi)$ and $\Hh(V\mid\widetilde V,\chi)$. Taking $p$ that approaches the capacity for one variable and upper-bounding the other by its capacity gives \cref{eq:control_gap}. The normalized identities are definitions of conditional entropy and capacity deficit. Vanishing bridge components give mutual recoverability of the quotients, transcript-identification of $Q$, lossless observation of $V$, and full reachability/capacity for $V$. Under these conditions the corresponding objectives attain their normalized maxima. Conversely, if the evaluated normalized objectives are jointly maximal, each nonnegative deficit in \cref{eq:gaps1,eq:gaps2} must vanish.

\subsection{Proof of \cref{thm:regret_transfer}}
Let $u_j^\star\in\arg\max_{u\in\mathcal U}J_j(u)$. Since $\hat u$ is $\eta$-optimal for $J_i$,
\begin{align}
J_j(u_j^\star)-J_j(\hat u)
&=[J_j(u_j^\star)-J_i(u_j^\star)]+[J_i(u_j^\star)-J_i(\hat u)]+[J_i(\hat u)-J_j(\hat u)]\\
&\le \eta+\sup_{u\in\mathcal U}(J_j(u)-J_i(u))+\sup_{u\in\mathcal U}(J_i(u)-J_j(u))\\
&=\eta+\Omega_{ij}(\mathcal U).
\end{align}
The bound is tight. Take two options $a,b$ with $J_i(b)=1,J_j(b)=1$ and $J_i(a)=1-\eta,J_j(a)=1-\eta-\Omega$, with the remaining values adjusted to keep objectives in $[0,1]$. Then $a$ is $\eta$-optimal for $J_i$, the oscillation of $J_i-J_j$ is $\Omega$, and $a$ has $J_j$ regret $\eta+\Omega$ whenever the parameters lie in the feasible range. If bridge-gap inequalities imply $|J_i(u)-J_j(u)|\le\varepsilon$ for all $u$, then $J_i-J_j\in[-\varepsilon,\varepsilon]$ pointwise and hence $\Omega_{ij}\le2\varepsilon$. If they instead bound the range width of $J_i-J_j$ by $\varepsilon$, then directly $\Omega_{ij}\le\varepsilon$.

\subsection{Proof of \cref{thm:fiber}}
Condition on a realized private seed $\Sigma=s$, or take the policy to be deterministic. For log loss, the Bayes predictor is the posterior distribution $p(Y\mid\Tcal_T,\Sigma=s)$ and its risk is $\Hh(Y\mid\Tcal_T,\Sigma=s)$. Because $Y=R_{T,s}^\pi(Z)$ and $\Tcal_T=G_{T,s}^\pi(Z)$ are deterministic functions of $Z$ under this conditioning, zero conditional entropy holds exactly when $Y$ is constant on every positive-prior fiber of $G_{T,s}^\pi$. Taking $Y=Z$ gives microstate-losslessness exactly when $G_{T,s}^\pi$ is injective on $\supp(p_0)$.

\subsection{Proof of \cref{cor:general_prediction}}
The quotient $Q_{\mathcal R}$ is the coarsest variable that determines every target in $\mathcal R$. Thus all targets in $\mathcal R$ are perfectly predictable from $\Tcal_T$ iff $Q_{\mathcal R}$ is. If $Q_{\mathcal R}=Z$ on the support, this is microstate-losslessness.

\subsection{Proof of \cref{lem:emp_reach}}
For fixed deterministic rollout context $\chi$, $V$ is a deterministic function of the action sequence. Hence $\I(A;V\mid\chi)\le\Hh(V\mid\chi)\le\log|\Reach_T^V(\chi)|$. Equality is achieved by choosing an action-sequence distribution whose induced distribution over the reachable values of $V$ is uniform, selecting one representative action sequence for each reachable value.

\subsection{Proof of \cref{prop:separation}}
For (i), let $Z$ be a uniform $n$-bit latent observed directly at $t=0$, and let actions have no effect. Identification is perfect and empowerment is zero. For (ii), let $Z$ be a uniform $n$-bit latent never observed, and let action set $\{0,1\}^m$ set an independent terminal distractor $D=A$. The transcript has zero information about $Z$. Empowerment over $D$ is $m$ bits. For (iii), let $Z$ be any latent, let the action overwrite the terminal target to a constant $V=0$, and let the agent remember only that it executed the overwrite. Then $\Hh(V\mid M_T)=0$ and $\I(Z;M_T)=0$.

\subsection{Proof of \cref{thm:missing_bits}}
If $\Hh(Q\mid\Tcal,B)=0$, then
\begin{equation}
\Hh(Q\mid\Tcal)=\I(Q;B\mid\Tcal)\le\Hh(B\mid\Tcal).
\end{equation}
For tightness, choose $B$ to be the posterior-fiber label $Q$ itself. Then $\Hh(Q\mid\Tcal,B)=0$ and $\Hh(B\mid\Tcal)=\Hh(Q\mid\Tcal)$.

\subsection{Proof of \cref{thm:refinement}}
Since $O_\eta$ is a deterministic function of $Z$ once actions and private seed are conditioned on, $\Hh(O_\eta\mid Z)=0$ and $\I(Z;O_\eta)=\Hh(O_\eta)=\Hh(Z)-\Hh(Z\mid O_\eta)$. If $\eta'$ refines $\eta$, then $O_\eta=r(O_{\eta'})$, so $Z\to O_{\eta'}\to O_\eta$ is a Markov chain and data processing gives $\I(Z;O_{\eta'})\ge\I(Z;O_\eta)$. Strict refinement on positive mass increases the induced partition entropy.

\subsection{Proof of \cref{thm:bgp_shaping}}
The intrinsic reward telescopes pathwise.
\begin{equation}
\sum_{t=0}^{T-1}\bigl(\Phi(H_t)-\Phi(H_{t+1})\bigr)=\Phi(H_0)-\Phi(H_T).
\end{equation}
Taking expectation gives \cref{eq:telescoping}. If task reward is zero, maximizing the left-hand side is equivalent to minimizing $\E\Phi(H_T)$ because $\Phi(H_0)$ is fixed. With task reward, the same identity gives
\begin{equation}
\E\sum_t r_t^{\rm BGP}=\E\sum_t r_t^{\rm task}+\beta\Phi(H_0)-\beta\E\Phi(H_T),
\end{equation}
which is the stated Lagrangian up to a constant. Finite-horizon dynamic programming returns an optimal policy for this additive reward. If a terminal zero-gap history is reachable and every higher-task-reward alternative remains below the chosen Lagrangian value, an optimal policy reaches zero gap.

\subsection{Proof of \cref{prop:algorithmic_separations}}
For (i), let $Q$ be a uniform $n$-bit latent that is hidden from the distractor action, and let an action sequence set an independent $m$-bit terminal variable $D$. Ungated empowerment over all terminal observations is $m$ bits for setting $D$. It leaves $\Hh(Q\mid H_T)=n$. A zero-gap option that inspects $Q$ has $\Hh(Q\mid H_T)=0$. The relevance gate closes for $D$ because $D\perp Q$ and $D$ has zero effect on task value.

For (ii), let one first action enter a channel state $C^\star$ and let the other stay outside it. First-step observations are independent of $Q$, so every first action has one-step information gain and one-step epistemic value zero. The second action reveals $Q$ only after the policy reaches $C^\star$. The BGP channel term decreases when the policy reaches $C^\star$, so exact BGP chooses the unique preliminary action that permits the later lossless experiment.

For (iii), let $Q$ be uniform and let the target for terminal prediction be $V$. One action reveals $Q$ and leaves $V=Q$. Another overwrites $V:=0$ and leaves $Q$ unrevealed. Both yield zero terminal log-loss for $V$, and a cost-tie-break can select overwrite. The overwrite policy has $\Hh(Q\mid M_T)=n$ and falls into the collapse branch of \cref{thm:absorption_trichotomy}. BGP selects inspection when the target quotient is $Q$ or when transfer to a $Q$-preserving target is evaluated.

For (iv), let $Q=(Q_c,Q_f)$ with $Q_f$ a uniform $n$-bit refinement independent of $Q_c$. A coarse training objective rewards only $Q_c$. Inspecting $Q_c$ achieves maximal training return and leaves $\Hh(Q_f\mid H_T)=n$. BGP evaluated on the refined quotient includes these missing bits in $\Phi$ and selects the refined inspection when reachable.

\subsection{Proof of \cref{thm:di_budget_main}}
Massey's conservation law for causal sequence distributions gives
\[
\I(A^T\to O^T)+\I(O^{T-1}\to A^T)=\I(A^T;O^T).
\]
If $A^T\perp O^T\mid B$, then $A^T\to B\to O^T$ is a Markov chain. Data processing gives $\I(A^T;O^T)\le \I(A^T;B)\le \Hh(B)\le b$.

\subsection{Proof of \cref{thm:absorption_trichotomy}}
If $\Hh(V\mid M_T)=0$, then $V$ is a deterministic function of $M_T$. Hence $Q\to M_T\to V$ is a Markov chain and data processing gives $\I(Q;M_T)\ge\I(Q;V)$. Also,
\begin{equation}
\Hh(Q\mid M_T)\le \Hh(Q,V\mid M_T)=\Hh(V\mid M_T)+\Hh(Q\mid V,M_T)\le \Hh(Q\mid V),
\end{equation}
where the first term is zero. If $\Hh(Q\mid V)=0$, then $\Hh(Q\mid M_T)=0$. The overwrite/collapse and omission cases occur when $V$ fails to preserve some distinctions in $Q$. The cause can be policy collapse or target-quotient omission.

\subsection{Proof of \cref{cor:bit_costs_main}}
If $\Hh(Q\mid M_T)=0$, then $\Hh(Q)=\I(Q;M_T)\le\Hh(M_T)\le\log|\mathcal M_T|$. The empowerment statement is \cref{lem:emp_reach}. Full authority from context $\chi$ means $|\Reach_T^V(\chi)|=|\supp(V\mid\chi)|$.

\section{Directed-information bridge budget}
\label{app:di_budget}

Use the same feedback-channel indexing as in \cref{sec:budget}: $A^T=(A_1,\ldots,A_T)$ and $O^T=(O_1,\ldots,O_T)$, where $A_t$ is chosen after $O^{t-1}$ and before $O_t$. Define
\begin{align}
  \I(A^T\to O^T)&=\sum_{t=1}^T \I(A^t;O_t\mid O^{t-1}),\label{eq:di_out}\\
  \I(O^{T-1}\to A^T)&=\sum_{t=1}^T \I(O^{t-1};A_t\mid A^{t-1}).\label{eq:di_in}
\end{align}
The first term is realized outward influence. The second is observation-to-action plasticity.

\begin{theorem}[Finite bridge budget]
\label{thm:di_budget}
Assume the joint law is causal and all statistical dependence between $A^T$ and $O^T$ is mediated by a specified finite bridge transcript $B$ satisfying $\Hh(B)\le b$, so $A^T\perp O^T\mid B$. The variable $B$ is the interface transcript being bounded. Then
\begin{equation}
  \I(A^T\to O^T)+\I(O^{T-1}\to A^T)=\I(A^T;O^T)\le b. \label{eq:budget}
\end{equation}
\end{theorem}

\begin{proof}
Massey's conservation law for causal sequence distributions states the displayed decomposition of mutual information. If $A^T\perp O^T\mid B$, then $A^T\to B\to O^T$ is a Markov chain. By data processing, $\I(A^T;O^T)\le \I(A^T;B)\le \Hh(B)\le b$.
\end{proof}

\section{Expanded deterministic bridge-interface model}
\label{app:expanded_model}

\paragraph{Finite deterministic microdynamics.}
The finite-state assumption avoids measure-theoretic pathologies of deterministic maps on continuous spaces and makes each bottleneck countable in bits. The environment microstate $X_t$ is a finite string or symbolic state. The law bits $\Lambda$ index a deterministic transition family,
\begin{equation}
  X_{t+1}=F_t(X_t,U_t,\Lambda), \qquad \Lambda_{t+1}=\Lambda_t=\Lambda . \label{eq:app_det_world}
\end{equation}
The prior over $Z=(X_0,\Lambda)$ is epistemic; exogenous stochasticity can be included by enlarging $Z$ to $(X_0,\Lambda,\Xi_{0:T})$. Once $Z$ and a realized private policy seed are fixed, the trajectory is deterministic. Ordinary stochastic POMDP kernels can be recovered by marginalizing over the environment variables and the policy seed. The deterministic refinement keeps track of which uncertainty could in principle be removed by a better transcript.

\paragraph{Actuation as a bridge.}
The intended action can differ from the realized physical control. Let $A_t^x$ denote the intended environment-targeting component of the action, and split the actuation bridge as
\begin{equation}
  \kappa_t=(\kappa_t^S,\kappa_t^X),\qquad \kappa_t^S=k_t^S(M_t),\quad \kappa_t^X=k_t^X(X_t,\Lambda).
\end{equation}
The realized control is
\begin{equation}
  U_t^x=\Gamma_t(A_t^x,\kappa_t^S,\kappa_t^X). \label{eq:app_actuation_bridge}
\end{equation}
Improving actuation can mean changing an agent-owned actuator setting, reaching a different world-side contact or permission condition, or identifying $\kappa_t^X$ well enough that the realized control is predictable. A deterministic actuator can still be lossy from the agent's perspective when $\kappa_t^X$ is hidden.

\paragraph{Sensing as a bridge.}
The observation map also depends on bridge state. Under the observe--act--observe timing convention, the agent chooses the next agent-owned sensor setting after $O_t$, and the world contributes the next channel state:
\begin{equation}
  O_{t+1}=h_{t+1}(X_{t+1},\phi_{t+1}^S,\phi_{t+1}^X),\qquad
  \phi_{t+1}^S=q_t^S(M_t,A_t),\quad \phi_{t+1}^X=q_{t+1}^X(X_{t+1},\Lambda). \label{eq:app_sensing_bridge}
\end{equation}
The initial observation $O_0$ is generated by a fixed initial setting or by a setting chosen before the episode. The same camera, probe, prompt, or API call can induce different experiments depending on line of sight, lighting, object pose, authorization, contact geometry, social context, or remote tool state. Losslessness is a property of the closed-loop transcript over a horizon. A pointwise property of $h_t$ alone is insufficient.

\paragraph{Finite-memory agent and learning.}
The agent-owned state $M_t$ includes fast state, persistent parameters, scratch memory, and any interface settings it can reliably maintain. The finite-capacity condition is
\begin{equation}
  |\mathcal{M}_t|\le 2^{n_s(t)},\qquad \Hh(M_t)\le n_s(t). \label{eq:app_finite_memory}
\end{equation}
Learning is modeled as internal dynamics,
\begin{equation}
  M_{t+1}=G_t(M_t,A_t,O_{t+1}), \label{eq:app_agent_update}
\end{equation}
This convention keeps learning inside the agent state and prevents an unbounded parameter store from bypassing the memory lower bounds. In the bit-budget specialization, all retained agent-owned degrees of freedom count toward $n_s(t)$.

\paragraph{Random variables and realizations.}
Uppercase symbols denote random variables induced by the prior over $Z$ and any private policy seed. Lowercase symbols denote realized values. Conditional on a realized $z$ and seed $s$, the full trajectory is deterministic. Entropies such as $\Hh(Z\mid\Tcal_T,\Sigma=s)$ are posterior uncertainties over possible deterministic worlds. They carry no implication of objective randomness in the world dynamics.

\section{Lossiness, coverage, and action-conditioned information gain}
\label{app:lossiness_coverage}

\paragraph{Instantaneous and horizon lossiness.}
The bridge model has two distinct bottlenecks. The instantaneous sensing loss is
\begin{equation}
  \mathcal{L}^{\rm sense}_t=\Hh(X_t\mid O_t,\phi_t^S,\phi_t^X). \label{eq:app_sense_loss}
\end{equation}
The online actuation loss is
\begin{equation}
  \mathcal{L}^{\rm act}_t=\Hh(U_t\mid H_t,A_t,\kappa_t^S), \label{eq:app_act_loss}
\end{equation}
where $H_t$ is the agent's decision-time information. This positivity can occur with a deterministic \cref{eq:app_actuation_bridge} because the environment-owned component $\kappa_t^X$ may be hidden at decision time. For identifiability, the relevant object is the horizon loss
\begin{equation}
  \mathcal{L}^{(T)}=\Hh(Z\mid O_{0:T},A_{0:T-1},\phi_{0:T}^S,\kappa_{0:T-1}^S). \label{eq:app_horizon_loss}
\end{equation}
This is the entropy form of microstate-losslessness in \cref{def:lossless}. Every per-step sensor may be individually lossy, and the full sequence can still be jointly separating.

\begin{assumption}[Interface coverage / refinability]
\label{ass:app_coverage}
For every $\varepsilon>0$ there exist a horizon $T$ and an admissible closed-loop design of actions and bridge settings such that $\mathcal{L}^{(T)}\le\varepsilon$.
\end{assumption}
This is an achievability assumption. It applies to interface families with separating experiments and policy classes that can execute them under memory and viability constraints.

\paragraph{Action-conditioned information gain.}
A raw objective that maximizes observation entropy is degenerate because the agent can create entropy by randomizing actions or injecting private randomness. The non-degenerate quantity conditions on the agent's chosen action and seed. For decision history $H_t$, define
\begin{equation}
  \mathrm{IG}_t=\I(Z;O_{t+1}\mid H_t,A_t). \label{eq:app_ig}
\end{equation}

\begin{lemma}[Information gain telescopes]
\label{lem:app_ig_telescopes}
Let $H_{t+1}=(H_t,A_t,O_{t+1})$. Then
\begin{equation}
  \mathrm{IG}_t=\Hh(Z\mid H_t,A_t)-\Hh(Z\mid H_{t+1}). \label{eq:app_ig_drop}
\end{equation}
If $A_t$ is a deterministic function of $H_t$, then $\Hh(Z\mid H_t,A_t)=\Hh(Z\mid H_t)$, and the cumulative gain satisfies
\begin{equation}
  \sum_{t=0}^{T-1}\mathrm{IG}_t=\Hh(Z\mid H_0)-\Hh(Z\mid H_T). \label{eq:app_ig_telescoped}
\end{equation}
\end{lemma}

\begin{proof}
The first identity is the definition of conditional mutual information. If $A_t$ is $H_t$-measurable, conditioning additionally on $A_t$ leaves the posterior over $Z$ unchanged. Summing the one-step posterior drops gives the telescoping identity.
\end{proof}

\begin{proposition}[Observation-compression progress equals hidden-state compression]
\label{prop:app_obs_progress}
In the deterministic bridge model, with actions and private seed conditioned on, maximizing cumulative action-conditioned observation novelty over a refinement-closed experiment class is equivalent to maximizing $\I(Z;H_T)$ and minimizing $\Hh(Z\mid H_T)$. If Assumption~\ref{ass:app_coverage} holds and the agent has sufficient memory to retain the latent-relevant posterior statistic, the optimum can approach microstate-losslessness.
\end{proposition}

\begin{proof}
Conditioning on actions and seed makes the transcript a deterministic function of $Z$. Hence \cref{thm:refinement} applies to each experiment and \cref{lem:app_ig_telescopes} identifies cumulative information gain with posterior entropy reduction. Sufficient memory allows the final agent state to preserve the same posterior partition as the history.
\end{proof}

\paragraph{Target quotients and safety-relevant ambiguity.}
For a target family $\mathcal{R}$, define $z\sim_{\mathcal{R}}z'$ if $R(z)=R(z')$ for every $R\in\mathcal{R}$. Robust prediction often needs the quotient $Z/{\sim_{\mathcal{R}}}$. The full microstate appears only in the limiting case where the target family is task-separating. This distinction is important for privacy and resource constraints. Bridge refinement should stay limited to authorized and task-relevant latent distinctions.

\section{Interface empowerment, bottlenecks, and state-mediated observation control}
\label{app:interface_empowerment}

\paragraph{Observation, state, and interface empowerment.}
The main text uses terminal-state empowerment. The same deterministic capacity form applies to any finite future variable. Here $x$ denotes a fully specified deterministic rollout context, including the law and exogenous bits needed for deterministic reachability. For a future variable $Y_T$ induced by an action sequence from $x$, define
\begin{equation}
  \Emp_T^{Y}(x)=\max_{p(a_{0:T-1})}\I(A_{0:T-1};Y_T\mid x).
\end{equation}
If $Y_T$ is a deterministic function of the action sequence conditional on $x$, this capacity equals the log-size of the reachable image of $Y_T$. We use three special cases,
\begin{align}
  \Emp_T^{O}(x)&=\log_2|\{O_T(a_{0:T-1};x):a_{0:T-1}\in\Acal^T\}|,\\
  \Emp_T^{X}(x)&=\log_2|\Reach_T(x)|,\\
  \Emp_T^{C}(x)&=\log_2|\{c(x'):x'\in\Reach_T(x)\}|.
\end{align}
The last quantity is interface empowerment. It measures authority over world-side channel states.

\begin{lemma}[Observation empowerment is bottlenecked by sensing]
\label{lem:app_obs_bottleneck}
Let $O_T=h_T(X_T,\phi_T^S,\phi_T^X)$ be deterministic given $X_T$ and bridge state. Then, for fixed initial $x$ and fixed admissible bridge schedule,
\begin{equation}
  \Emp_T^{O}(x)\le \Emp_T^{X}(x). \label{eq:app_obs_bottleneck}
\end{equation}
Equality requires the observation map to be injective on the reachable terminal states used by a capacity-achieving action distribution.
\end{lemma}

\begin{proof}
The deterministic channel from action sequences to $O_T$ factors through $X_T$ and bridge state. Under a fixed bridge schedule, the image of $O_T$ has cardinality at most the image of $X_T$. The capacity form as log image size gives the result.
\end{proof}

\paragraph{State-mediated observation control.}
Directly writing an observation register gives poor evidence about the world. Define state-mediated observation control as the portion of observation reachability that passes through a change in environment state. A bypass channel is excluded by the Markov factorization
\begin{equation}
  A_{0:T-1}\longrightarrow X_T\longrightarrow O_T
\end{equation}
conditional on initial state and known bridge settings. Under this condition, data processing gives
\begin{equation}
  \I(A_{0:T-1};O_T\mid x)\le \I(A_{0:T-1};X_T\mid x). \label{eq:app_state_med_dpi}
\end{equation}
Thus progress in controllable observations can come from enlarging the reachable set of relevant world states or refining the sensing bridge so that distinctions in $X_T$ survive into $O_T$.

\paragraph{Why interface empowerment is necessary when refinement requires steering.}
Suppose there is a set $\Ccal^\star$ of world-side channel states in which the sensing experiment is lossless for the target quotient, and outside $\Ccal^\star$ every admissible experiment leaves posterior entropy at least $\ell>0$. If $\Emp_T^C(x)=0$ and the current channel state is outside $\Ccal^\star$, every policy lacks guaranteed access to a lossless experiment by horizon $T$. A design sequence that drives $\Hh(Z\mid\Tcal_T)\to0$ in environments where losslessness is possible only after entering $\Ccal^\star$ must make $C_t$ action-dependent along the relevant histories. This is the formal sense in which prediction pressure creates target-conditioned interface empowerment.

\begin{proposition}[Compression pressure implies interface empowerment under steering-only refinability]
\label{prop:app_progress_interface_emp}
Assume Assumption~\ref{ass:app_coverage}. Further assume that every sequence of experiments with $\Hh(Z\mid\Tcal_T)\to0$ must reach a set of world-side channel states $\Ccal^\star$ from which lossless or asymptotically lossless experiments are available. If the uncontrolled channel process stays outside $\Ccal^\star$ on the relevant support, then any admissible design sequence achieving $\Hh(Z\mid\Tcal_T)\to0$ must have nonzero interface empowerment $\Emp_T^C(x)>0$ for the corresponding initial states.
\end{proposition}

\begin{proof}
If $\Emp_T^C(x)=0$, varying the action sequence leaves the reachable channel-state value unchanged. Starting outside $\Ccal^\star$, all reachable experiments are bounded away from lossless by assumption, so terminal posterior entropy remains positive. This contradicts the assumed achievement of $\Hh(Z\mid\Tcal_T)\to0$.
\end{proof}

\section{Full authority, overwrite, and conditional absorption}
\label{app:authority_absorption}

\paragraph{Full authority requires target reachability.}
A deterministic action consequence means $\Hh(X_{t+1}\mid X_t,A_t,\kappa_t,\Lambda)=0$. Determinism alone can still give small empowerment if all actions map to the same next state. Authority requires a large reachable set.

\begin{definition}[State-conditioned full authority]
\label{def:app_state_authority}
For fixed law bits and horizon $T$, the actuator has state-conditioned full authority if for every $x_0\in\Xcal$ and every target $x^\star\in\Xcal$ there exists an action sequence $a_{0:T-1}$ such that $F_T(x_0,a_{0:T-1},\Lambda)=x^\star$.
\end{definition}

\begin{definition}[Strong overwrite authority]
\label{def:app_overwrite_authority}
The actuator has strong overwrite authority if for every target $x^\star\in\Xcal$ there exists an action sequence $a_{0:T-1}$ such that $F_T(x_0,a_{0:T-1},\Lambda)=x^\star$ for every $x_0$ in the prior support.
\end{definition}
State-conditioned authority is ordinary reachability conditional on knowing the state. Strong overwrite authority makes the terminal microstate independent of the unknown initial state.

\paragraph{Using state-conditioned authority requires identification.}
Let $B_t=\supp(X_t\mid H_t)$ be the set of current states compatible with the agent's information. State-conditioned authority becomes unusable when different states in $B_t$ require different action sequences.

\begin{proposition}[Usable state-conditioned authority requires eliminating control-separating ambiguity]
\label{prop:app_authority_needs_ident}
Assume state-conditioned full authority. If there exist $x,x'\in B_t$ and a target $x^\star$ such that every admissible action sequence fails to drive at least one of $x$ or $x'$ to $x^\star$, then an agent whose policy depends only on $H_t$ lacks a guarantee of reaching $x^\star$. Agent-level full authority therefore requires identifying the control-separating component of $X_t$ or possessing strong overwrite authority.
\end{proposition}

\begin{proof}
The policy must choose the same action sequence on histories with the same $H_t$. By assumption, that sequence fails for at least one of the two possible states. Hence the guarantee fails for at least one state in $B_t$.
\end{proof}

\paragraph{Bit-string specialization.}
Assume a finite conserved bit representation
\begin{equation}
  n_s(t)+n_x(t)=n_{\rm tot}, \qquad |\mathcal{M}_t|\le 2^{n_s(t)},\quad |\Xcal_t|=2^{n_x(t)}. \label{eq:app_fixed_bits}
\end{equation}
All retained agent-owned bits count toward $n_s(t)$, including parameters, scratch memory, externalized memory under the agent's control, and agent-owned bridge settings.

\begin{assumption}[Ownership/writability convention]
\label{ass:app_ownership}
When a degree of freedom is counted as agent-owned, it is treated as writable and maintainable by the agent's internal transition dynamics. Boundary shift is therefore an architectural assumption about which bits the agent can reliably retain and reuse. Observation alone leaves physical objects outside the controllable set.
\end{assumption}

\begin{theorem}[Conditional absorption and unification]
\label{thm:app_absorption}
Assume the deterministic bridge model, the conserved bit budget \eqref{eq:app_fixed_bits}, and Assumption~\ref{ass:app_ownership}. Consider an ideal sequence of agents and horizons.
\begin{enumerate}
  \item If the identification route achieves $\Hh(Z\mid M_T)=0$, then $n_s(T)\ge\Hh(Z)$. If $Z=(X_0,\Lambda)$ has entropy $n_{\rm tot}-o(n_{\rm tot})$ and $n_s(0)=o(n_{\rm tot})$, then $n_x(T)=n_{\rm tot}-n_s(T)=o(n_{\rm tot})$.
  \item If the overwrite route achieves strong authority over $\Xcal_T$ with $|\Xcal_T|=2^{n_x(T)}$ and per-step action alphabet size at most $2^{b_t}$, then $\sum_{t<T}b_t\ge n_x(T)$ and deterministic empowerment over $X_T$ equals $n_x(T)$.
  \item In both routes, the remaining effective bridge bottleneck must vanish relative to the target quotient. Identification requires a transcript map injective on the quotient. Overwrite requires an action channel spanning the quotient's terminal degrees of freedom.
\end{enumerate}
\end{theorem}

\begin{proof}
The first item is the memory lower bound in \cref{thm:bit_budget}. Under the stated entropy and initial-budget conditions, $n_s(T)=n_{\rm tot}-o(n_{\rm tot})$. The environment allocation is $o(n_{\rm tot})$. The second item is the action-counting lower bound in \cref{thm:bit_budget} and the deterministic reachability-capacity identity in \cref{lem:emp_reach}. The third item restates the necessary bottlenecks exposed by the two lower bounds. The transcript must distinguish the relevant latent quotient, or the action channel must set the terminal quotient directly.
\end{proof}

\begin{remark}[Scope of the absorption statement]
\label{rem:app_absorption_scope}
The theorem is intentionally conditional. The result requires a conserved bit budget, boundary-shift architecture, and ownership/writability convention to support an absorption conclusion. It remains a memory or action-capacity lower bound under less restrictive assumptions. It is a resource statement and carries no claim of physical inevitability. Its role is to show that exact microstate prediction and exact overwrite control have microstate-scale resource requirements.
\end{remark}

\paragraph{Self-modeling limit and plasticity.}
In the boundary-collapse limit of \cref{thm:app_absorption}, the latent state is a function of the agent-owned state. If the policy and update rule are deterministic, the coupled system can be represented as a closed map $M_{t+1}=U(M_t)$. Observations then add zero novel information relative to $M_t$. The incremental observation-to-action plasticity associated with novel evidence vanishes. It concerns novelty rate and leaves responsiveness to already-internalized state as a distinct property.

\section{Empirical instantiations, expanded related work, and broader impacts}
\label{app:empirical_related}

\paragraph{Exact benchmark instantiations.}
The exact benchmark in \cref{tab:exact_results} enumerates the four option families. In the settable-distractor benchmark, $Q\sim\mathrm{Unif}(\{0,1\}^n)$ and an independent $D\in\{0,1\}^m$ is writable. Ungated empowerment chooses $D$. BGP closes the relevance gate and inspects $Q$. In the delayed-sensor benchmark, the first action only changes channel reachability. Immediate information gain is zero for all first actions. BGP values the transition to $C^\star$. In the inspect-overwrite benchmark, terminal log-loss for $V$ is zero under inspection and under overwrite. Inspection preserves $Q$. In quotient transfer, training on $Q_c$ hides the missing $Q_f$ bits until the refined target is evaluated.

Useful metrics are transcript ambiguity $\Hh(Z\mid\Tcal_T)$, target ambiguity $\Hh(Y\mid\Tcal_T)$, reachability $|\Reach_T(x)|$, interface reachability $|\{c(x'):x'\in\Reach_T(x)\}|$, and the directed-information split in \cref{eq:budget}. Reporting return or prediction loss alone hides the mechanism of success.

\paragraph{Evaluation probes.}
A bridge-aware evaluation should include counterfactual probes. Hold the transcript fixed and vary latent state when possible. Vary world-side channel conditions with the agent-side query fixed. Remove settable distractors. Make informative channel states harder to reach. Restrict action bandwidth to test whether predictive competence collapses when overwrite control is unavailable. In learned systems where exact fibers are unavailable, contrastive latent probes, intervention tests, and bandwidth ablations are empirical approximations to the finite quantities in the theorems.

\paragraph{Broader impacts.}
\label{app:broader_impacts}
The positive use of the theory is diagnostic. It asks whether an agent's control authority is over task-relevant variables, whether its sensors can identify safety-relevant hidden state, and whether the bridge between human intent, internal state, tools, and external world is auditable. The same concepts could be misused to design agents that more effectively acquire information, obtain permissions, or manipulate interfaces. Practical deployments should therefore pair interface-refinement objectives with access control, monitoring of tool use, privacy constraints, and limits on control over sensitive world-side channel conditions.


\end{document}